\documentclass[12pt,a4paper]{amsart}
\usepackage{graphicx}
\usepackage{subfigure}
\usepackage{amsmath,amssymb,amsfonts}

\newcounter{thm}
\newcounter{ex}
\newtheorem{Proposition}[thm]{Proposition}
\newtheorem{Example}[ex]{Example}

\newcommand{\Id}{id}

\DeclareMathOperator{\Diff}{Diff}
\DeclareMathOperator{\SO}{SO}
\DeclareMathOperator{\Emb}{Emb}
\DeclareMathOperator{\Gr}{Gr}
\DeclareMathOperator{\GL}{GL}

\def \Jet {\mathcal{J}}
\newcommand{\R}{\ensuremath{\mathbb{R}}}

\title{Symmetry in Image Registration and Deformation Modeling}
\author{Stefan Sommer and Henry O. Jacobs}

\begin{document}
\maketitle

\begin{abstract}
  We survey the role of symmetry in diffeomorphic 
  registration of landmarks, curves, surfaces, images and higher-order data. 
  The infinite dimensional problem of finding correspondences
  between objects can for a range of concrete data types be reduced resulting in 
  compact representations of shape and spatial structure.
  This reduction is possible because the available data
  is incomplete in encoding the full deformation model. Using reduction by
  symmetry, we describe the reduced models in a common theoretical framework 
  that draws on links between the registration
  problem and geometric mechanics. Symmetry also arises in reduction to the Lie 
  algebra using particle relabeling symmetry allowing the equations of motion to be
  written purely in terms of Eulerian velocity field. Reduction by
  symmetry has recently been applied for
  jet-matching and higher-order discrete approximations of the image
  matching problem. We outline these constructions and 
  further cases where reduction by
  symmetry promises new approaches to registration of complex data types.
\end{abstract}


\section{Introduction}
Registration, the task of establishing correspondences between multiple
instances of objects such as images, landmarks, curves, and surfaces, plays a
fundamental role in a range of computer vision applications including shape
modeling \cite{younes_shapes_2010}, motion
compensation and optical flow \cite{brox_high_2004}, remote sension
\cite{dawn_remote_2010}, and medical imaging \cite{sotiras_deformable_2013}. In
the subfield of computational anatomy \cite{younes_evolutions_2009}, establishing inter-subject correspondences 
between organs allows the statistical study of organ shape and shape variability. 
Examples of the fundamental role of registration include quantifying developing Alzheimer's disease 
by establishing correspondences between brain tissue at different
stages of the disease \cite{boyes_cerebral_2006}; measuring the effect of COPD on lung tissue after
removing the variability caused by the respiratory process
\cite{gorbunova_early_2010}; and correlating the 
shape of the hippocampus to schizophrenia after inter-subject registration
\cite{joshi_geometry_1997}.

In this paper, we survey the role of symmetry in diffeomorphic
registration and deformation modeling and link symmetry as seen from the
field of geometric mechanics with the image registration problem. We focus on large deformations
modeled in subgroups of the group of diffeomorphic mappings on the spatial
domain, the approach contained in the Large Deformation Diffeomorphic Metric
Mapping (LDDMM,
\cite{dupuis_variational_1998,trouve_infinite_1995,christensen_deformable_2002,younes_shapes_2010}) 
framework. Connections with geometric
mechanics \cite{holm_soliton_2004} have highlighted the 
role of symmetry and resulted in previously known
properties connected with the registration of specific data types being
described in a common theoretical framework \cite{jacobs_symmetries_2013}. We wish to describe these connections
in a form that highlights the role of symmetry and points towards future
applications of the ideas. It is the aim that the paper will make the role of symmetry in
registration and deformation modeling clear to the reader that has no
previous familiarity with symmetry in geometric mechanics and symmetry groups in
mathematics.

\subsection{Symmetry and Information}
One of the main reasons symmetry is useful in numerics
is in it's ability to reduce how much information one must carry.
As a toy example, consider the a top spinning in space.
Upon choosing some reference configuraiton,
the orientation of the top is given by a rotation matrix,
i.e. an element $R \in \SO(3)$.
If I ask for you to give me the direction of the pointy tip of the top,
(which is pointing opposite $\mathbf{k}$ in the reference)
it suffices to give me $R$.
However, $R$ is contained in space of dimension $3$,
while the space of possible directions is the $2$-sphere, $S^2$,
which is only of dimension $2$.
Therefore, providing the full matrix $R$ is excessive
in terms of data.
It suffices to just provide the vector $R \cdot \mathbf{k} \in S^2$.
Note that if $\tilde{R} \cdot \mathbf{k} = \mathbf{k}$,
then $R \cdot \mathbf{k} = R\cdot \tilde{R} \cdot \mathbf{k}$.
Therefore, given only the direction $\mathbf{k}' = R \cdot \mathbf{k}$,
we can only reconstruct $R$ up to an element $\tilde{R}$ which
preserves $\mathbf{k}$.
The group of element which preserve $\mathbf{k}$ is identifiable
with $\SO(2)$.
This insight allows us to express the space of directions $S^2$
as a homogenous space $S^2 \equiv \SO(3) / \SO(2)$.
In terms of infomation we can cartoonishly express this by the 
expression
\begin{align*}
  \text{``orientation''} = 
  \text{``direction of tip''} + 
  \text{``orientation around the tip''}
\end{align*}

This example is typically of all group quotients.  If $X$ is some universe
of objects and $G$ is a group which acts freely upon $X$, then 
the orbit space $X / G$ hueristically contains the data of $X$
minus the data which $G$ transforms.  Thus
\begin{align*}
  \text{data}(X) = \text{data}(X/G) + \text{data}(G).
\end{align*}
Reduction by symmetry can be implemented when a problem posed on $X$
has $G$ symmetry, and can be rewritten as a problem posed on $X/G$.
The later space containing less data, and is therefore more efficient
in terms of memory.

\subsection{Symmetry in Registration}
Registration of objects contained in a spatial domain, e.g. the
volume to be imaged by a scanner, can be formulated as the search for a
deformation that transforms both domain and objects to establish an inter-object match.
The data available when solving a registration problem generally is incomplete
for encoding the deformation of every point of the domain. This is for example the
case when images to be matched have areas of constant intensity and no
derivative information can guide the registration. Similarly, when 3D shapes 
are matched based on similarity of their surfaces, the deformation of the interior 
cannot be derived from the available information.
The deformation model is in these cases
over-complete, and a range of deformations can
provide equally good matches for the data. Here arises \emph{symmetry}: the
subspaces of deformations for which the registration problem is symmetric with
respect to the available information. When quotienting out symmetry
subgroups, a vastly more compact representation is obtained. In the image case,
only displacement orthogonal to the level lines of the image is needed; in the
shape case, the information left in the quotient is supported on the surface of
the shape only.
\begin{figure}[t]
  \begin{center}
    \subfigure[fixed image]{
      \includegraphics[width=.25\columnwidth,trim=85 55 155 55,clip]{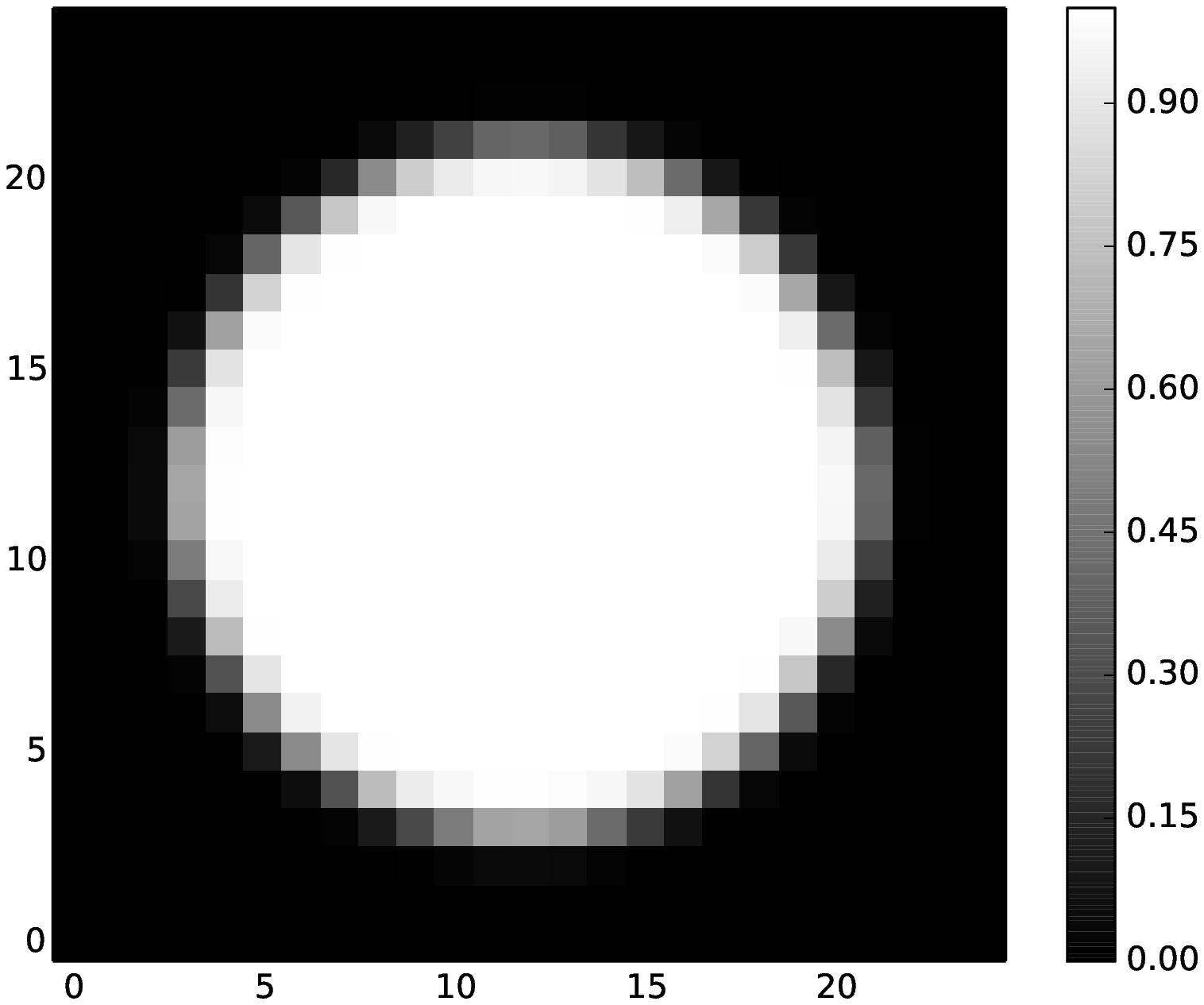}
    }
    \subfigure[moving image]{
      \includegraphics[width=.25\columnwidth,trim=85 55 155 55,clip]{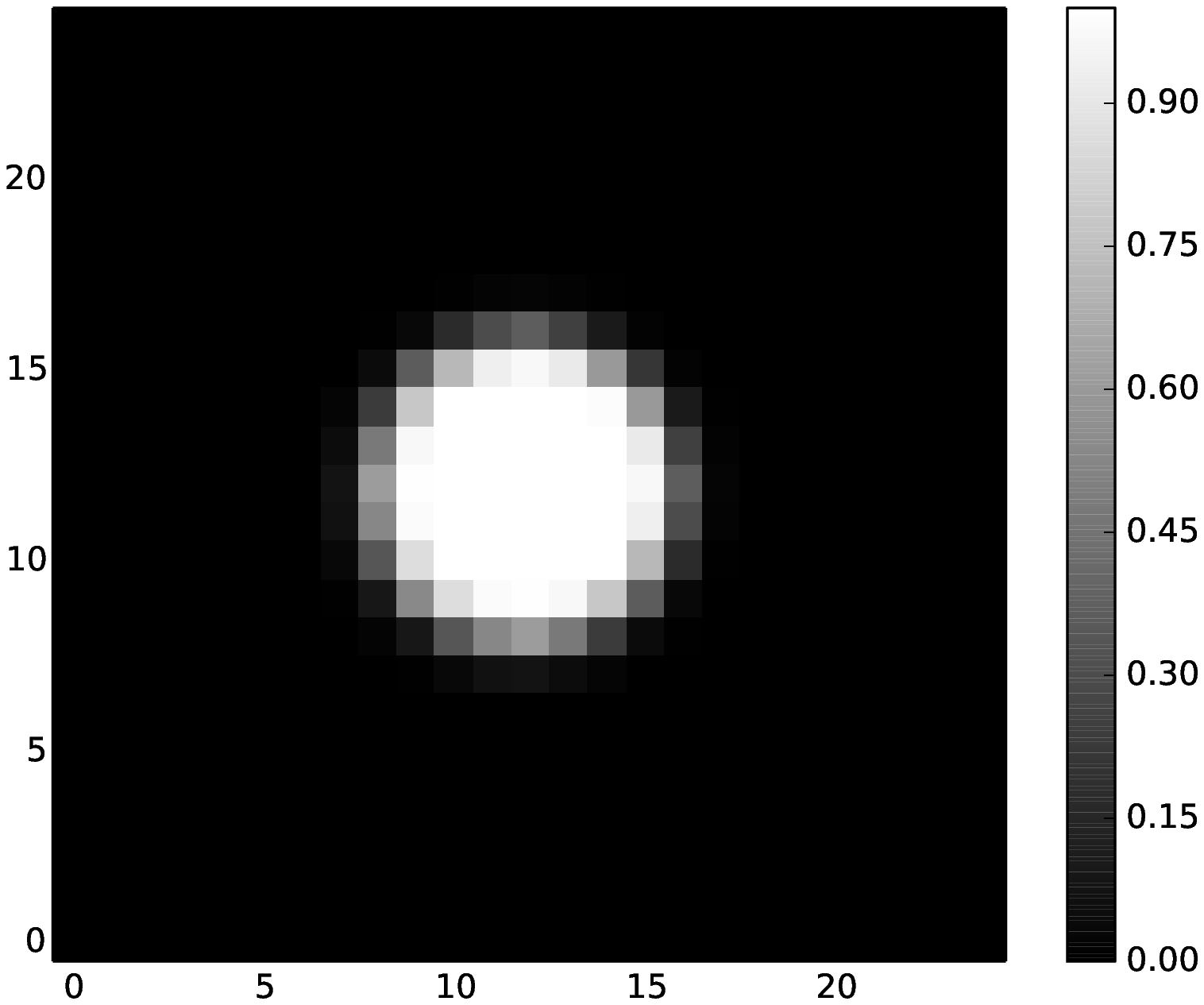}
    }
    \subfigure[warp]{
      \includegraphics[width=.25\columnwidth,trim=130 56 115 50,clip]{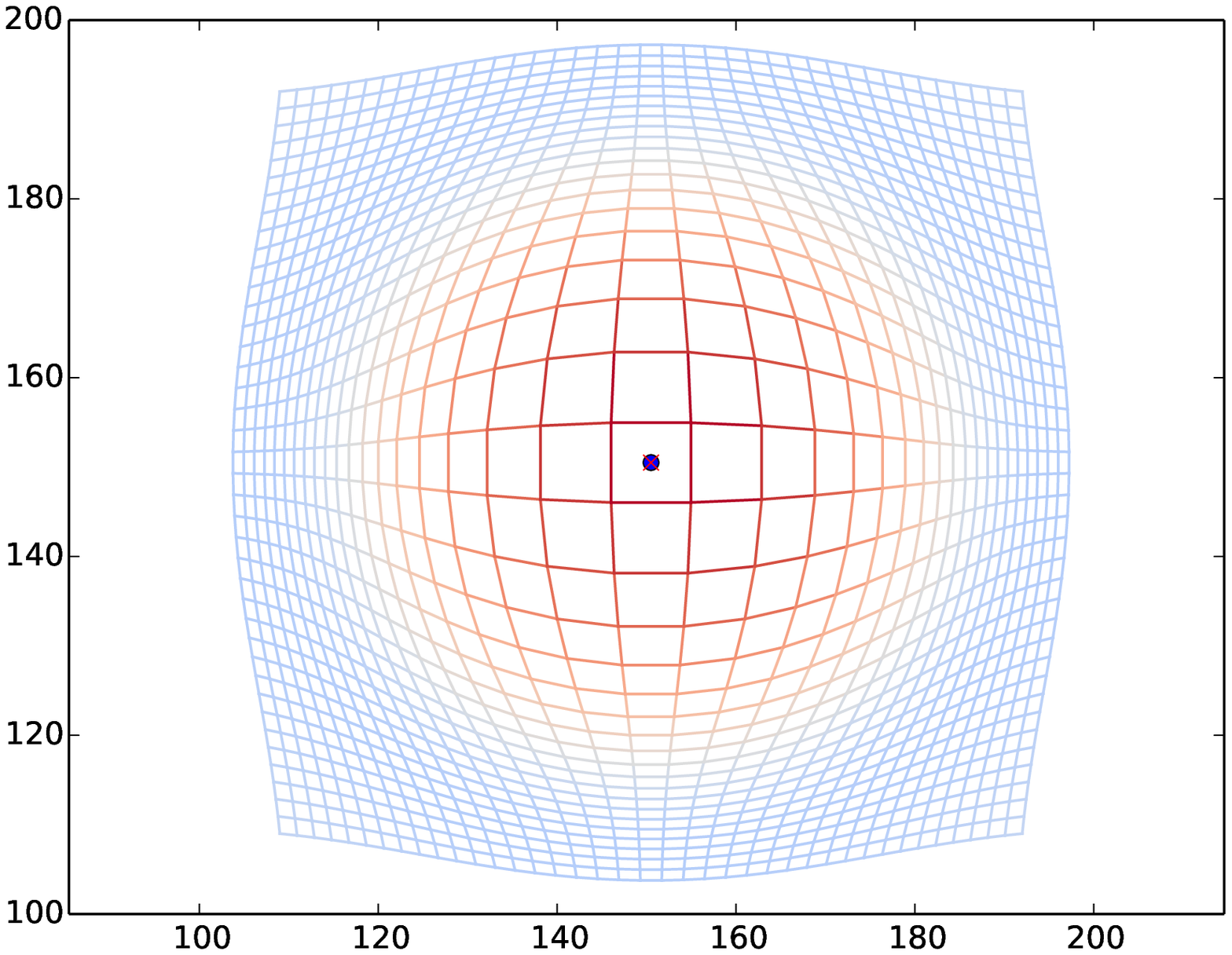}
    }
  \end{center}
  \caption{A registration of two discs of different sizes (a,b) with a
    warp that brings (b) into correspondence with (a) visualized by its effect on
    an initially regular grid (c). Using symmetry,
    the dimensionality of the registration problem can be reduced from infinite to finite. In this case, 6
    parameters of a 1-jet particle in the center of the moving image encode the entire deformation.
    }

  \label{fig:intro}
\end{figure}

\subsection{Content and Outline}
We start with background on the registration problem and the large deformation
approach from a variational viewpoint.
Following this, we describe how reduction by symmetry leads to an Eulerian formulation
of the equations of motion when reducing to the Lie
algebra. Symmetry of the dissimilarity measure allows additional reductions,
and we use isotropy subgroups to reduce the complexity of the registration problem
further. Lastly, we survey the effect of symmetry in a range of concrete registration
problems and end the paper with concluding remarks.

\section{Registration and Variational Formulation}
The registration problem consists in finding correspondences between objects
that are typically point sets (landmarks), curves, surfaces, images or more
complicated spatially dependent data such as diffusion weighted images (DWI).
The problem can be approached by
letting $M$ be a spatial domain containing the objects to be registered. $M$ can
be a differentiable manifold or, as is often the case in applications,
the closure of an open subset of $\R^d$, $d=2,3$, e.g. the unit square. A map
$\varphi:M\rightarrow M$ can deform or warp the domain by mapping each
$x\in M$ to $\varphi(x)$.
\begin{figure}[t]
  \begin{center}
      \includegraphics[width=.50\columnwidth,trim=800 0 00 0,clip]{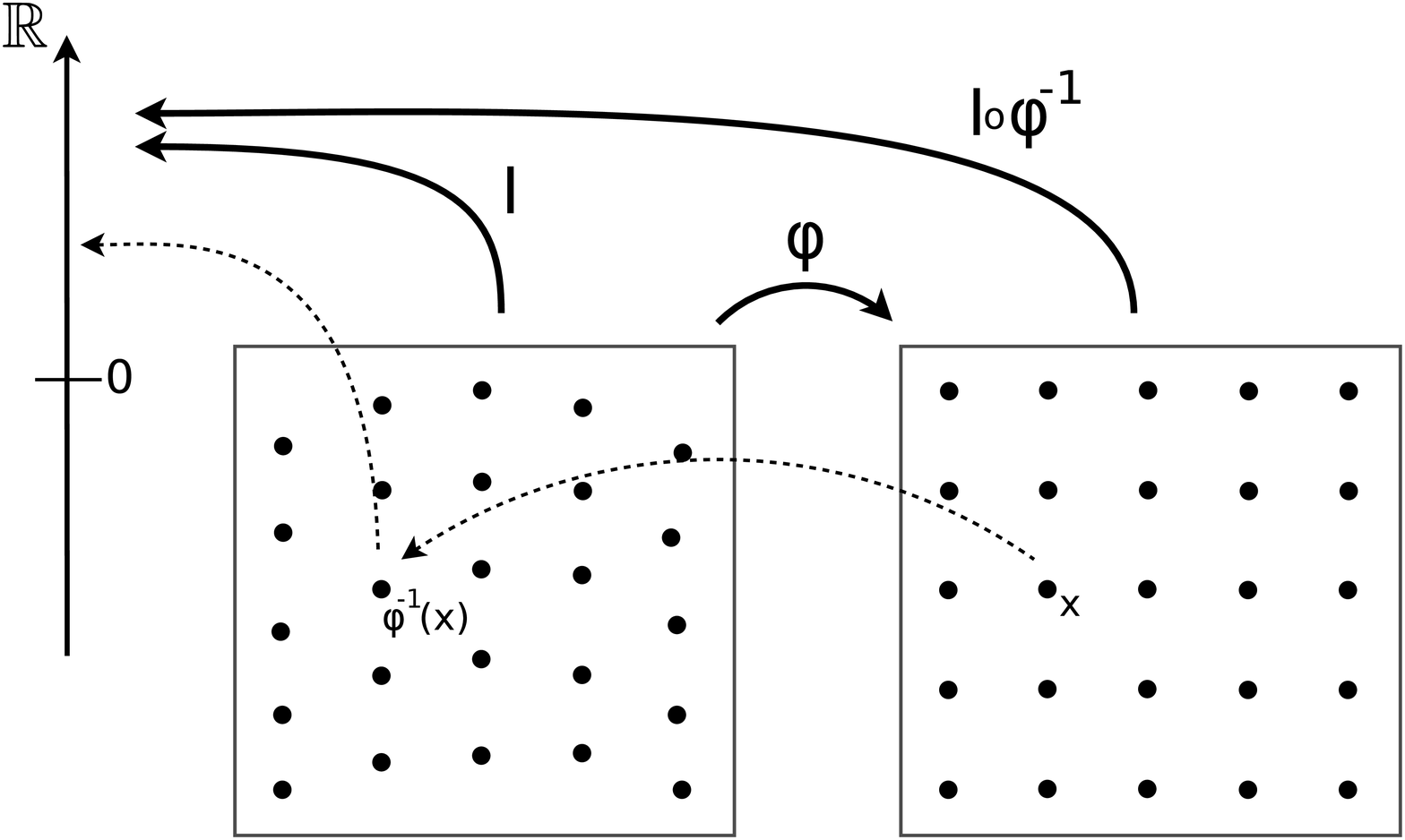}
  \end{center}
  \caption{A warp $\varphi\in\Diff(M)$ acts on an image $I:M\rightarrow\R$ by
    composition with the inverse warp,
    $\varphi\cdot I=I\circ\varphi^{-1}$. Given two images $I_0,I_1:M\rightarrow\R$, image
    registration involves finding a warp $\varphi$ such that $\varphi\cdot I_0$ is close to
    $I_1$ as measured by a dissimilarity measure $F(\varphi\cdot I_0,I_1)$.
    }

  \label{fig:registration}
\end{figure}

The deformation encoded in the warp will apply to the objects in $M$ as well 
as the domain itself. For example,
if the objects to be registered consist of points sets $\{x_1,\ldots,x_N\}$,
$x_i\in M$, the set will be mapped to
$\{\varphi(x_1),\ldots,\varphi(x_N)\}$. For surfaces $S\subset M$, $\varphi$
similarly results in the warped surface $\varphi(S)$. Because those operations
are associative, the mapping $\varphi$
acts on $\{x_i\}$ or $S$ and we write $\varphi\cdot\{x_i\}$ and $\varphi\cdot S$ for the warped
objects. An image is a function $I:M\rightarrow\R$, and $\varphi$ acts on $I$
as well, in this case by composition with its inverse $\varphi\cdot
I=I\circ\varphi^{-1}$,
see Figure~\ref{fig:registration}. For this
$\varphi$ must be is invertible, and commonly we restrict to the set of
invertible and differentiable mappings $\Diff(M)$. For various other types of data objects, 
the action of a warp on the objects can be defined in a way similar to the case for point
sets, surfaces and images. This fact relates a range registration problems to
the common case of finding appropriate warps $\varphi$ that trough the action
brings the objects into correspondence. Trough the action, different instances of 
a shape can be realized by letting warps act on a base instance of the shape, and
a class of shape models can therefor be obtained by using deformations to
represent shapes \cite{younes_shapes_2010}.

\subsection{Variations over Warps and Families of Warps}
The search for appropriate warps can be formulated in a variational formulation
with an energy 
\begin{equation}
  E(\varphi)
  =
  R(\varphi)
  +
  F(\varphi)
  \label{eq:var}
\end{equation}
where $F$ is a dissimilarity measure of the difference between the deformed
objects, and
$R$ is a regularization term that penalizes unwanted properties of $\varphi$ such as
irregularity. If two objects $o_1$ and $o_2$ is to be matched, $F$ can take the
form $F(\varphi\cdot o_0,o_1)$ using the action of $\varphi$ on $o_0$; for image matching,
an often used dissimilarity measure is the
$L^2$-difference or sum of square differences (SSD) that has the form
$F(\varphi\cdot I_0,I_1)=\int_M|I_0\circ\varphi^{-1}(x)-I_1(x)|^2dx$.

The regularization term can take various forms often modeling physical properties
such as elasticity \cite{pennec_riemannian_2005} and penalizing derivatives of
$\varphi$ in order to make it
smooth. The free-form-deformation (FFD, \cite{rueckert_nonrigid_1999}) and related approaches penalize
$\varphi$ directly. For some choices of $R$, 
existence and analytical properties of minimizers of \eqref{eq:var} have been
derived \cite{derfoul_relaxed_2014}, however it is
in general difficult to ensure solutions are diffeomorphic by penalizing
$\varphi$ in itself.
Instead, flow based approaches model one-parameter families or paths of mappings
$\varphi_t$, $t\in[0,1]$ where $\varphi_0$ is the identity mapping $\Id\in\Diff(M)$ and the
dissimilarity is measured at the endpoint $\varphi_1$. The time evolution of
$\varphi_t$ can be described by the differential equation
$\frac{d}{dt}\varphi_t(x)=v_t(\varphi(x))$ with the flow field $v_t$ being 
a vector field on $M$. The space of such fields is denoted
$V$. In the Large Deformation Diffeomorphic Metric Mapping (LDDMM,
\cite{younes_shapes_2010})
framework, the regularization is applied to the flow field
$v_t$ and integrated over time giving the energy
\begin{equation}
  E(\varphi_t)
  =
  \int_0^1\|v_t\|_V^2dt
  +
  F(\varphi_1)
  \label{eq:varlddmm}
  \ .
\end{equation}
If the norm $\|\cdot\|_V$ that measures the irregularity of $v_t$ is
sufficiently strong, $\varphi_t$ will be a diffeomorphism for all $t$. This
approach thus gives a direct way of enforcing properties of
the generated warp: Instead of regularizing $\varphi$ directly, the analysis is
lifted to a normed space $V$ that is much easier control. 
The energy \eqref{eq:varlddmm} has the same minimizers as the geometric
formulation of LDDMM used in the next section.

Direct approaches to solving the optimization problem \eqref{eq:varlddmm} must handle
the fact that the problem of finding a warp is now transfered to finding a time-dependent 
family of warps implying a huge increase in dimensionality. This problem is
therefore vary hard to represent numerically and to optimize. For several data
types, it has been shown how optimal paths for \eqref{eq:varlddmm} have specific
properties that reduces the dimensionality of the problem and therefor makes
practical solutions feasible. In the next section, we describe the
geometric framework and the reduction theory that allows the data dependent
results to be formulated as specific examples of reduction by symmetry. We
survey these examples in the following section.

\section{Reduction by symmetry in LDDMM}
\label{sec:red}
We here describe a geometric formulation of the registration problem \cite{younes_shapes_2010}
and how symmetry can be used to reduce the optimization over
time-dependent paths of warps to vector fields in the Lie algebra resulting
in an Eulerian version of the equations of motion. Secondly, we describe how symmetry of the dissimilarity
measure allows further reduction to lower dimensional quotients.

\subsection{Kinematics}
We here introduce a number of notions from differential
geometry in a fairly informal manner.
For formal definitions we refer to \cite{FOM}.
While it is neccessary for the purpose of rigour to learn formal
definitions, independent of cartoonish sketches, when learning differential geometry,
one can still get quite far with cartoonish sketches.
For example, by picturing a manifold, $M$,
as a surface embedded in $\R^3$.
This is the approach we will take.

The tangent bundle of $M$, denoted $TM$, is the set of pairs $(x,v)$
where $x \in M$ and $v$ is a vector tangential to $M$
at the point $x$ (see Figure~\ref{fig:tm}).
A vector-field is a map $u:M \to TM$ such that $u(x) \in TM$
is a vector above $x$ for all $x \in M$.
In summary, a vector on $M$ is an ``admissible velocity''
on $M$, and $TM$ is the set of all possible admissible velocities.
\begin{figure}[t]
  \begin{center}
    \includegraphics[width=.30\columnwidth,trim=0 50 0 0,clip]{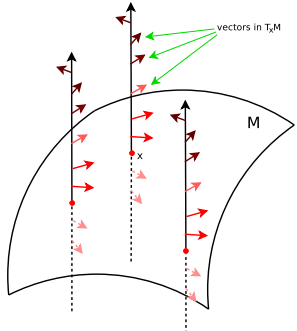}
  \end{center}
  \caption{The tangent bundle $TM$ of the manifold $M$ consists of pairs $(x,v)$ of
  points $x\in M$ and tangent vectors $v\in T_xM$. It's a fiber bundle over $M$ with
    fibers $T_xM$ for each $x\in $M.}
  \label{fig:tm}
\end{figure}

Given a vector-field $u$ we may consider the initial value problem
\begin{align*}
  \begin{cases}
    x(0) = x_0 \\
    \frac{dx}{dt} = u( x(t))
  \end{cases}
\end{align*}
for $t \in [0,1]$.
Given an intial condition $x_0$, the point $x_1 = x(1)$ given by 
solving this initial value problem is uniquely determined (if it exists).
Under reasonable conditions $x_1$ exists for each $x_0$,
and there is a map $\Phi_t^{u}: x_0 \in M \mapsto x_1 \in M$ which we call the flow
of $u$.
If $u$ is time-dependent we can consider the initial value problem with 
$\frac{dx}{dt} =  u(t,x(t))$.  Under certain conditions, this will also yield a flow map, $\Phi^u_{t_0,t_1}$
which is the flow from time $t=t_0$ to $t=t_1$.
If $u$ is smooth, the flow map is smooth as well, in particular a
diffeophism.  We denote the set of diffeomorphisms by $\Diff(M)$.

Conversely, let $\varphi_t \in \Diff(M)$
be a time-dependent diffeomorphism.
Thus, for any $x \in M$,
we observe that $\varphi_t(x)$
is a curve in $M$.
If this curve is differentiable
we may consider its time-derivative,
$\frac{d \varphi_t}{dt}(x)$,
which is a vector above the point $\varphi_t(x)$.
From these observations it imediately follows that
$\frac{d \varphi_t}{dt} (\varphi_t^{-1}(x))$
is a vector above $x$.
Therefore the map $u(t) : M \to TM$, given by
$u(t) := \left(\frac{d}{dt} \varphi_t \right) \circ \varphi_t^{-1}$
is a vector-field which we call the \emph{Eulerian velocity field of $\varphi_t$}.

The Eulerian velocity field contains less data than
$\frac{d\varphi_t}{dt}$, and this reduction in data
can be viewed from the perspective of symmetry.
Given any $\psi \in \Diff(M)$,
the curve $\varphi_t$ can be transformed to the curve
$\varphi_t \circ \psi$.
We observe that
\begin{align*}
  u(t) &:= \frac{d\varphi}{dt} \circ \varphi^{-1} = \frac{d \varphi_t}{dt} \circ \psi \circ \psi^{-1} \circ \varphi^{-1} \\
  &= \left( \frac{d}{dt}( \varphi_t \circ \psi) \right) \circ (\varphi\circ \psi)^{-1}.
\end{align*}
Thus $\varphi_t$ and $\varphi_t \circ \psi$ both have the same Eulerian velocity fields.
In other words, the Eulerian velocity field, $u(t)$, is invariant under particle relablings.
Schematically, the following holds
\begin{align*}
  \text{data} \left( \frac{d\varphi}{dt} \right) = \text{data}(u(t)) + \text{data}(\varphi_t).
\end{align*}

Finally, we will denote some linear operators on the space of vector-fields.
Let $\Phi \in \Diff(M)$ and let $u \in \mathfrak{X}(M)$.
The \emph{push-forward} of $u$ by $\Phi$, denoted $\Phi_* u$,
is the vector-field given by
\begin{align*}
  [\Phi_* u](x) = \left. D\Phi \right|_{\Phi^{-1}(x)} \cdot u( \Phi^{-1}(x)).
\end{align*}
By inspection we see that $\Phi_*$ is a linear operator on the vector-space of vector-fields.
One can view $\Phi_* u$ as  ``$u$ in a new coordinate system'' because any differential 
geometric property of $u$ is also inherited by $\Phi_* u$.
For example, if $u(x) = 0$ then $[\Phi_* u] ( y) = 0$ with $y = \Phi(x)$.
If $S$ is an invariant under $u$ then $\Phi(S)$ is invariant under $\Phi_*u$.

As $\Phi_*$ is linear, we can transpose it.
Let $\mathfrak{X}(M)^*$ denote the dual space to the space of vector-fields,
i.e. the set of linear maps $\mathfrak{X}(M)\rightarrow\R$,
and let $m \in \mathfrak{X}(M)^*$.
We define $\Phi_*m \in \mathfrak{X}(M)^*$ by the equality
\begin{align*}
  \langle \Phi_*m , \Phi_*v \rangle = \langle m , v \rangle
\end{align*}
for all $v \in \mathfrak{X}(M)$ where $\langle m , v \rangle$ denotes evaluation
of $m$ on $v$.

Finally, we define the Lie-derivative as the linear operator
$\pounds_w : \mathfrak{X}(M) \to \mathfrak{X}(M)$
defined by
\begin{align*}
  \pounds_w[ u ] = \left. \frac{d}{d \epsilon}\right|_{\epsilon=0} (\Phi^w_{\epsilon})_* u.
\end{align*}
As $\pounds_w$ is linear, we can take its transpose.
If $m \in \mathfrak{X}(M)^*$, then we can define $\pounds_u[m] \in \mathfrak{X}(M)^*$ by the equation
\begin{align*}
  \langle \pounds_u[m] , w \rangle + \langle m , \pounds_u[w] \rangle = 0
\end{align*}
for all $w \in\mathfrak{X}(M)$.
This is a satisfying because for a fixed $m$ and $w$ we observe
\begin{align}
 \langle \pounds_u[m] , w \rangle + \langle m,\pounds_u[w] \rangle = 0 = \frac{d}{dt} \langle m , w \rangle = \frac{d}{dt}\langle (\Phi_t^u)_* m , (\Phi_t^u)_* \rangle
\label{eq:product_rule}
\end{align}
This is nothing but a coordinate free version of the product rule.

\subsection{Reduction to Lie Algebra}
The variational formulation \eqref{eq:varlddmm} of LDDMM is equivalent to
minimizing the energy
\begin{align*}
  E = d( id , \varphi) + F( \varphi)
\end{align*}
where $d: \Diff(M) \times \Diff(M) \to \mathbb{R}$
is a distance metric on $\Diff(M)$,
$id$ is the identity diffeomorphism,
and $F:\Diff(M) \to \mathbb{R}$ is a function
which measures the disparity between the deformed template and the target image.

\begin{Example}
Given images $I_0, I_1 \in L^2(M)$, we consider the dissimilarity measure
\begin{align*}
  F(\varphi) = \| (I_0 \circ \varphi^{-1}) - I_1 \|_{L^2(M)}^2.
\end{align*}
\label{ex:imssd}
\end{Example}
In this article we will consider the distance metric
\begin{align*}
  d(\varphi_0,\varphi_1) = \inf_{ 
    \substack{
      v \in C^0([0,1] , \mathfrak{X}(M) ) \\
      \Phi^v_{0,1} \circ \varphi_0 = \varphi_1
      }
    } \left( \int_0^1 \| v(t) \| dt \right),
\end{align*}
where $\| \cdot \|$ is some norm on $\mathfrak{X}(M)$.
If $\| \cdot \|$ is induced by an inner-product, then this distance metric is
(formally) a Riemannian distance metric on $\Diff(M)$.
Note that the distance metric, $d$, is written in terms of a norm $\| \cdot \|$, defined on $\mathfrak{X}(M)$.
In fact, the norm on $\mathfrak{X}(M)$ induces a Riemannian metric on $\Diff(M)$ given by
\begin{align*}
	\| \frac{d\varphi_t}{dt} \|^2 := \| \frac{d\varphi_t}{dt} \circ \varphi_t^{-1} \|^2,
\end{align*}
and $d$ is the Reimannian distance with respect to this metric.
If the norm $\| \cdot \|$ imposes a Hilbert space structure on the vector-fields
it can be written in terms of a psuedo-differential operator $P: \mathfrak{X}(M) \to \mathfrak{X}(M)^*$ as $\| u \|^2 = \langle P[u] , u \rangle$ \cite{younes_shapes_2010}.

Given $P$, minimizers of $E$ must neccessarily satisfy
\begin{align}
  \begin{cases}
  m(t) = (\Phi^u_{t,1})_* m(1) = P[u(t)]\\
  \langle m(1)  , w \rangle = \left. \frac{d}{d\epsilon} \right|_{\epsilon=0} F( \Phi^w_\epsilon \circ \Phi^u_{0,1}) \quad,\quad \forall w \in \mathfrak{X}(M).
  \end{cases} \label{eq:extreme1}
\end{align}
This is a vector-calculus statement of Proposition 11.6 of \cite{younes_shapes_2010}.
That this equation of motion is even well-posed is nontrivial, since $P$ is merely an injective map,
and there is no guarantee that it can be inverted to obtain a vector-field $u$ to integrate into a diffeomorphism.
Fortunately, safety guards for well-posedness are studied in \cite{TrouveYounes2005}.
If the reproducing kernel of $P$ is $C^1$, then the equations of motion are well-posed for all time.

There is something unsatisfying about using \eqref{eq:extreme1} for the purpose of computation.
Doing any sort of computation on $\Diff(M)$ is difficult, as it is a nonlinear infinite dimensional space.
Moreover, the dissimilarity measure $F$ only comes into play at time $t=1$ and the distance function 
is an integral over the vector-space $\mathfrak{X}(M)$.  It would be nice if we could rewrite the extremizers
purely in terms of the Eulerian velocity field, $u$ and the flow at $t=1$.
In fact this is often the case.
Given $P$, the minimizer of $E$ must neccessarily satisfy the boundary value problem
\begin{align}
  \begin{cases}
  \partial_t m + \pounds_u[m] = 0, m = P[u] \quad \forall t \in [0,1] \\
  \frac{d}{d \epsilon}|_{\epsilon = 0} \left[ F( \Phi^w_\epsilon \circ \Phi^u_{0,1} ) \right] +  \langle P[u(1)] , w \rangle = 0 \quad, \forall w \in \mathfrak{X}(M).
  \end{cases} \label{eq:extreme2}
\end{align}
This is an alternative formulation of \eqref{eq:extreme1}, obtained
simply by taking a time-derivative.
In the language of fluid-dynamics, \eqref{eq:extreme2} is an Eulerian version of \eqref{eq:extreme1}.
The advantage of this formulation, is that the bulk of the computation occurs on the vector-space $\mathfrak{X}(M)$,
and this observation is the starting point for the algorithm given in \cite{beg_computing_2005}.

This reduction of the problem to the space of vector-fields is a first instance of reduction
by symmetry.  In particular, this corresponds to the fact that the space of vector-fields
$\mathfrak{X}(M)$, is identifiable as a quotient space
\begin{align*}
  \mathfrak{X}(M) \equiv T\Diff(M) / \Diff(M).
\end{align*}
And the map $(\varphi,\frac{d\varphi}{dt}) \in T\Diff(M) \mapsto \frac{d\varphi_t}{dt} \circ \varphi^{-1} \in \mathfrak{X}(M)$.
is the quotient projection.

\subsection{Isotropy Subgroups}
The reduction to dynamics on $\Diff(M)$ to dynamics on $\mathfrak{X}(M)$ occurs primarly 
because the distance function is $\Diff(M)$ invariant.
However, one can not completely abandon $\Diff(M)$ because the solution requires one to
compute the time $1$ flow, $\Phi^u_{0,1}$.
Fortunately, there is a second reduction which allows us to avoid computing $\Phi^u_{0,1}$ in
its entirety.
This second reduction corresponds to the invariance properties of the
dissimilarity measure $F$.
Let $G_F \subset \Diff(M)$ denote the set of diffeomorphisms which
leave $F$ invariant, i.e.:
\begin{align*}
  G_F := \{ \psi \in \Diff(M) \mid F( \varphi \circ \psi ) = F(\varphi),  \forall \varphi \in \Diff(M) \}.
\end{align*}
One can readily verify that $G_F$ is a subgroup of $\Diff(M)$,
and so we call $G_F$ the \emph{isotropy subgroup of $F$}.

Having defined $G_F$ we can now consider the homogenous space $Q = \Diff(M) / G_F$,
which is the \emph{quotient space} induced by the action of right composition of $G_F$
on $\Diff(M)$.  This quotient space is ``smaller'' in the sense of data.
In terms of maps, this can be seen by defining the map $\varphi \in \Diff(M) \mapsto q = [\varphi]_{/G_F} \in Q$, where $[\varphi]_{/G_F}$ denotes the equivalence class
of $\varphi$.  We call this mapping the \emph{quotient projection} because it sends $\Diff(M)$ to $Q$ surjectively.
While these notions are theoretically quite complicated, they often manifest more
simply in practice.

\begin{Example}\label{ex:two_particles}
  Let $M \subset \mathbb{R}^n$ be the closure of some open set.
  Let $x_1,x_2,y_1,y_2 \in M$ with $x_1 \neq x_2$ and consider the dissimilarity measure
  \begin{align*}
    F(\varphi) = \| \varphi(x_1) - y_1 \|^2 + \| \varphi(x_2) - y_2 \|^2.
  \end{align*}
  We see that
  \begin{align*}
    G_F \equiv \{ \psi \in \Diff(M) \mid \psi(x_1) = x_1, \psi(x_2) = x_2 \},
  \end{align*}
  and
  \begin{align*}
    Q = \Diff(M) / G_F \equiv \{ (z_1,z_2) \in M \times M \mid z_1 \neq z_2 \} = M \times M - \Delta_{M \times M},
  \end{align*}
  where $\Delta_{M \times M}$ denotes the diagonal of $M \times M$.
  The quotient projection is $\varphi \in \Diff(M) \mapsto ( \varphi(x_1) , \varphi(x_2)) \in M \times M - \Delta_{M \times M}$.
  Note that $\Diff(M)$ is infinite dimensional while $Q$ is of dimension $2 \dim(M)$.
  This is massive reduction.
\end{Example}
\begin{figure}[t]
  \begin{center}
      \includegraphics[width=.25\columnwidth,trim=140 80 130 80,clip]{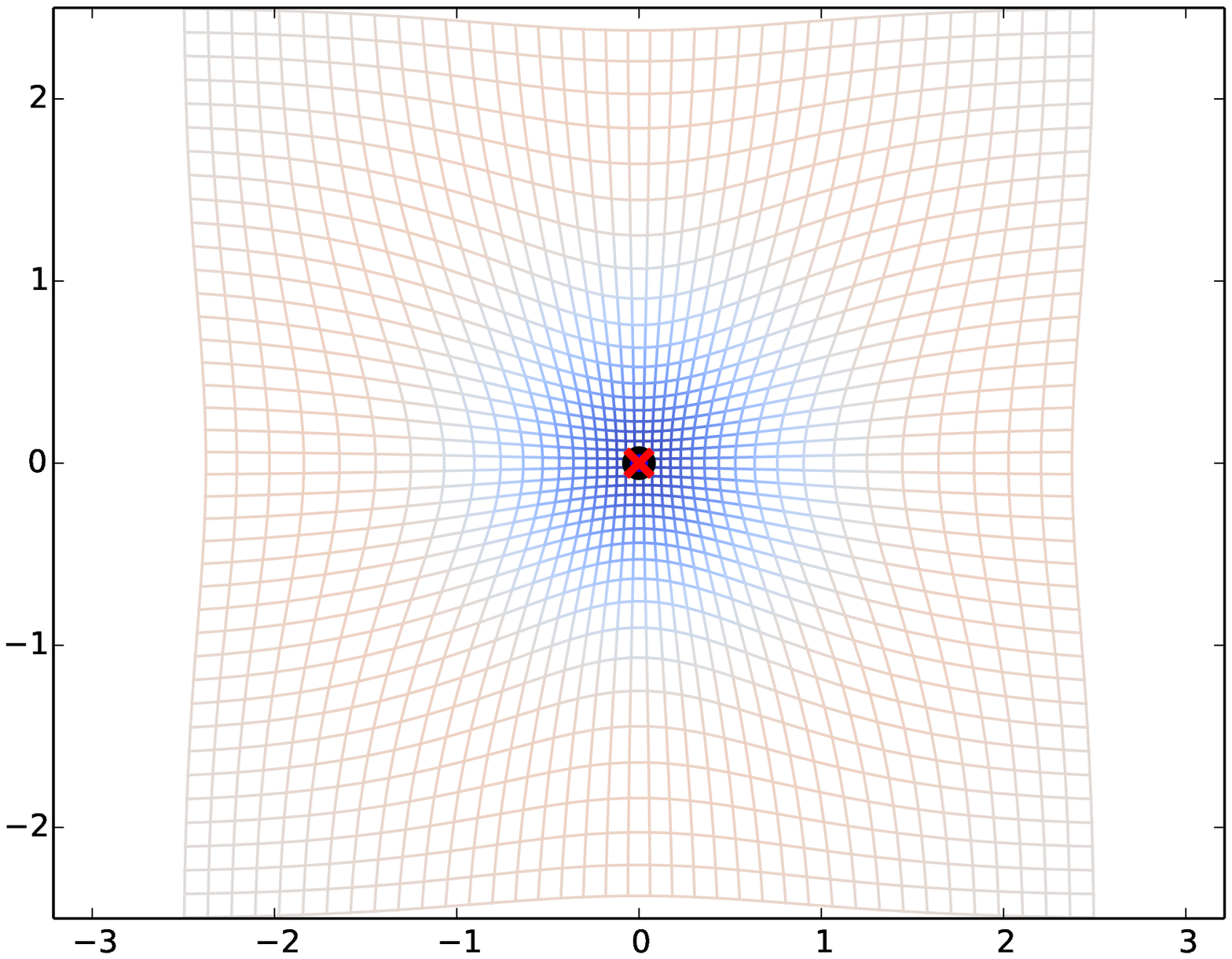}
      \includegraphics[width=.25\columnwidth,trim=140 80 130 80,clip]{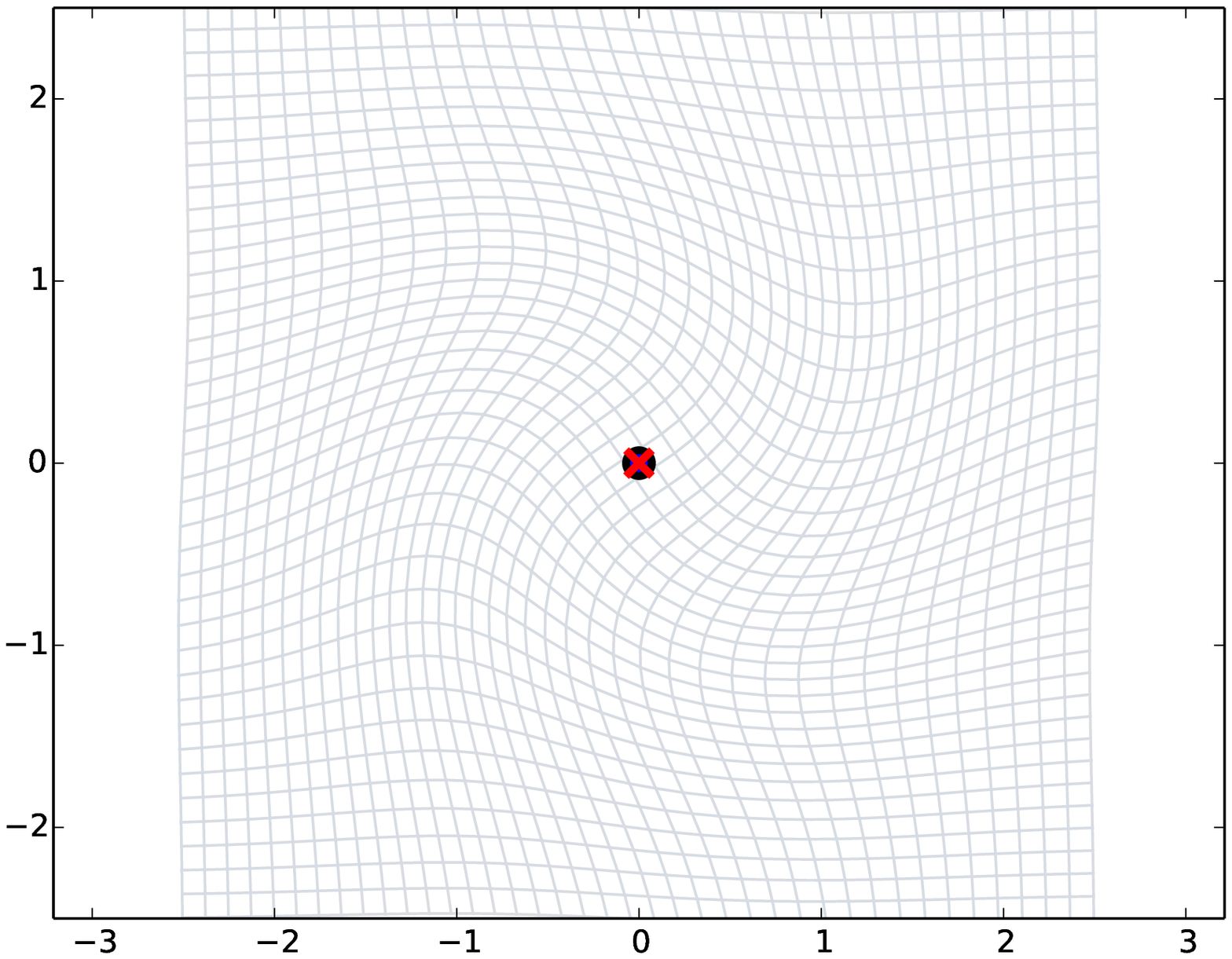}
      \includegraphics[width=.25\columnwidth,trim=140 80 130 80,clip]{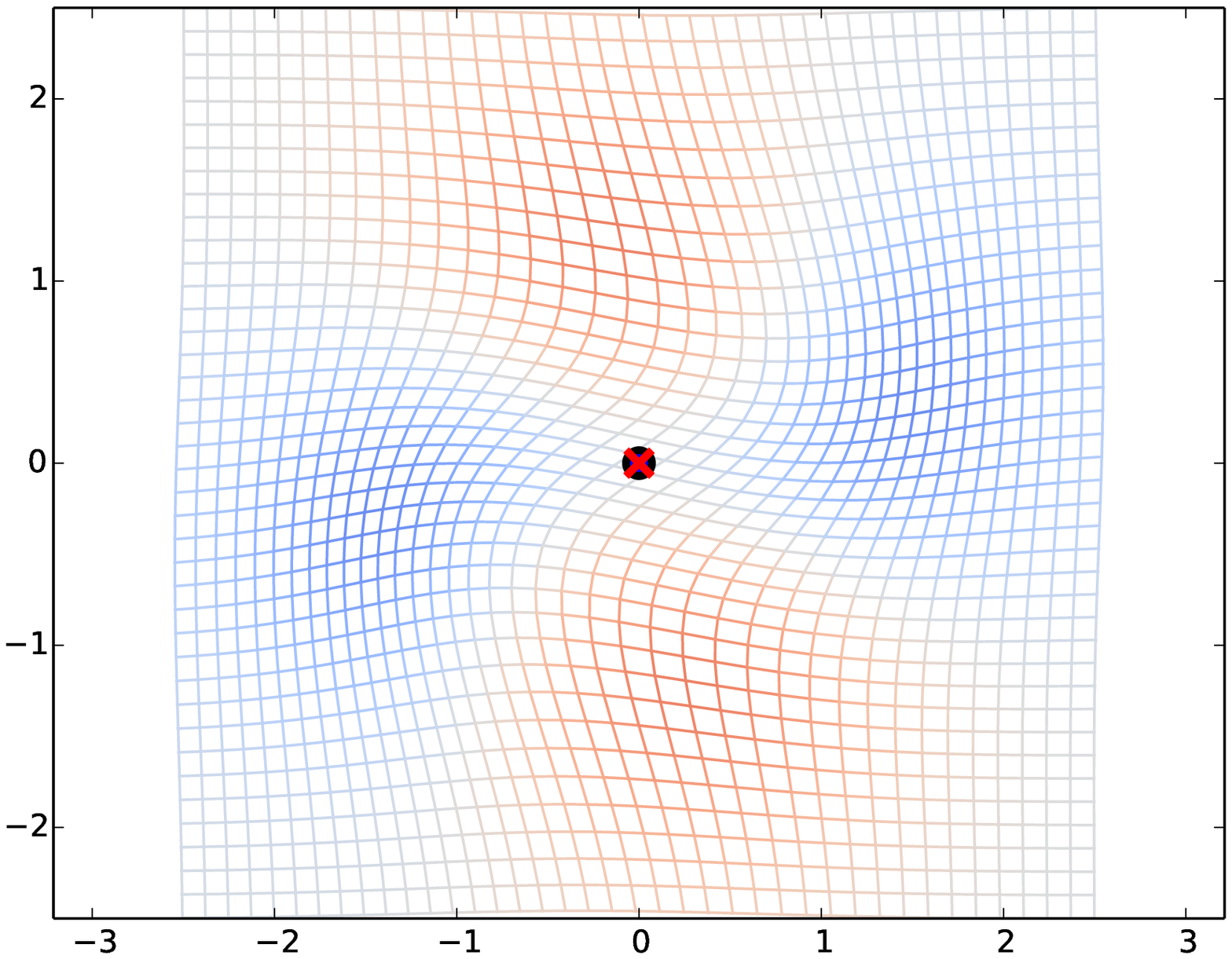}
  \end{center}
  \caption{Examples of elements of the isotropy subgroup $\{ \psi \in \Diff(M) \mid \psi(x) = x\}$
     for a one-point matching problem with dissimilarity measure $F(\varphi) =
     \| \varphi(x) - y \|^2$ visualized by their effect on an initially square
     grid. The isotropy subgroup leaves $F$ invariant by not moving $x$.
    }
  \label{fig:isotropy}
\end{figure}

If one is able to understand $Q$ then one can use this insight
to reformulate the dissimilarity measure $F$ as 
function on $Q$ rather than $\Diff(M)$.
In particular, there neccessarily exists a unique function $F_{Q} : Q \to \mathbb{R}$
defined by the property $F_{Q}( [\varphi]_{/G_F} ) = F(\varphi)$.
Again, this is useful in the sense of data, as is illustrated in the following example.

\begin{Example}
  Consider the dissimilarity measure $F$ of example \ref{ex:two_particles}.
  The function, $F_{Q}:Q \to \mathbb{R}$ is
  \begin{align*}
    F_{Q}( z_1,z_2) = \| z_1 - y_1 \|^2 + \| z_2 - y_2 \|^2.
  \end{align*}
\end{Example}

Finally, note that $\Diff(M)$ acts upon $Q$ by the left action
\begin{align*}
  [\varphi]_{G_F} \in \Diff(M) / G_F \stackrel{\psi \in \Diff(M) }{\longmapsto} [\psi \circ \varphi]_{/G_F} \in \Diff(M) / G_F.
\end{align*}
Usually we will simply write $\psi \cdot q$ for the action of $\psi \in \Diff(M)$ on a 
given $q \in Q$.
This means that $\mathfrak{X}(M)$ acts upon $Q$ infinitesimally, as it is the Lie
algebra of $\Diff(M)$.
\begin{Example}
  Consider the setup of example \ref{ex:two_particles}.
  Here $Q = M \times M - \Delta_{M \times M}$ and the left action of $\Diff(M)$
  is given by
  \begin{align*}
    \psi \cdot (q_1,q_2) = ( \psi(q_1) , \psi(q_2) )
  \end{align*}
  for $\psi \in \Diff(M)$ and $q = (q_1,q_2) \in Q$.
  The infinitesimal action of $u \in \mathfrak{X}(M)$ on $Q$ is
  \begin{align*}
    u \cdot (q_1,q_2) = (u(q_1) , u(q_2) ) \in T_qQ.
  \end{align*}
\end{Example}

These constructions allow us to rephrase the initial optimization problem using a reduced
curve energy.  Minimization of $E$ is equivalent to minimization of
\begin{align*}
  E_{Q} = \int_{0}^{1} \| u \| dt + F_{Q}\left(q(1) \right)
\end{align*}
where $q(1)$ is obtained by integrating the ODE, $\frac{dq}{dt} = u \cdot q$
with the intial condition $q(0) = [\Id]_{/G_F}$ where $\Id \in \Diff(M)$ is the identity transformation.
We see that this curve energy only depends on the Eulerian velocity field 
and the equivalence class $q(1)$.
Minimizers of $E_{Q}$ must neccessarily satisfy
\begin{align}
  \begin{cases}
  \partial_t m + \pounds_u[m] \quad , \quad m = P[u] \\
  \langle u(1)  , w \rangle = - DF(q) \cdot (w \cdot q) \quad,\quad \forall w \in \mathfrak{X}(M).
  \end{cases} \label{eq:extreme3}
\end{align}
Again, the solution only depends on the Eulerian velocity and $q(1)$.
For this reason, we see that the $G_F$ symmetry of $F$ provides a second reduction in
the data needed to solve our original problem.

\subsection{Orthogonality}
In addition to reducing the amount of data we must keep track of 
there is an additional consequence to the $G_F$-symmetry of $F$.
In particular, there is a potentially massive constraint satisfied
by the Eulerian velocity $u$.

To describe this we must introduce an isotropy algebra.
Given $q(t) = [\Phi^u_{0,t}]_{/G_F}$
we can define the (time-dependent) isotropy algebra
\begin{align*}
  \mathfrak{g}_{q(t)} = \{ w \in \mathfrak{X}(M) \mid w \cdot q(t) = 0 \}.
\end{align*}
This is nothing but the Lie-algebra associated to the isotropy group
$G_{q(t)} = \{ \psi \in \Diff(M) \mid \psi \cdot q(t) = q(t) \}$.

It turns out that the velocity field $u(t)$ which minimizes $E$ (or $E_{Q}$)
is orthogonal to $\mathfrak{g}_{q(t)}$ with respect to the chosen inner-product.
Intuitively this is quite sensible because velocities
which do not change $q(t)$ do not alter the data, and
simply waste control effort.
This intuitive statement is roughly the content of the following proof.

\begin{Proposition}
  Let $u$ satisfy \eqref{eq:extreme2} or \eqref{eq:extreme3}.
  Then $m = P[u]$ anihillates $\mathfrak{g}_{q(t)}$.
\end{Proposition}
\begin{proof}
  Let $u$ be the solution to \eqref{eq:extreme3}.
  We will first prove that $u(1)$ (this is $u$ at time $t=1$)
  is orthogonal to $\mathfrak{g}_{z(1)}$.
  Let $w(1) \in \mathfrak{g}_{z(1)}$.
  We observe
  \begin{align*}
    \langle P[u(1)] , w(1) \rangle \stackrel{\text{by \eqref{eq:extreme3}}}{=}
    - \left.\frac{d}{d\epsilon}\right|_{\epsilon=0} F_{Q}( \Phi^{w(1)}_\epsilon z(1) ). 
  \end{align*}
  However, $w(1)$ leaves $z(1)$ fixed, so $\Phi^{w(1)}_\epsilon \cdot q(1) = 0$.
  Therefore $\langle P[u(1)] , w(1) \rangle = 0$.
  Let  $w(t) = [\Phi^u_{t,1}]^*w(1)$
  In coordinates this means
  \begin{align*}
    w^i(t,x) = \left.\partial_j\right|_{[\Phi^u_{t,1}]^{-1}(x)}
    [\Phi^u_{t,1}]^i w^j \left(1, [\Phi^u_{t,1}]^{-1}(x) \right)
  \end{align*}
  One can directly verify that $w(t) \in \mathfrak{g}_{z(t)}$ for all $t \in [0,1]$.
  Denoting $m(t) = P[u(t)]$, as in \eqref{eq:extreme3}, we find
  \begin{align*}
    \frac{d}{dt} \langle P[u(t)] ,  w(t) \rangle &= 
    \frac{d}{dt} \langle m(t) , w(t) \rangle 
    \langle \partial_t m , w \rangle + \langle m , \partial_t w \rangle \\
    &=\langle - \pounds_u [m] , w \rangle + \langle m , -\pounds_u[w] \rangle = 0.
  \end{align*}
  Where the last equality follows from \eqref{eq:product_rule}.
  Thus $\langle P[u(t)] , w(t) \rangle$ is constant.
  We've already verified that at $t=1$, this inner-product is zero, thus
  $\langle P[u(t)] , w(t) \rangle = 0$ for all time.
  That $w(1)$ is an arbitrary element of $\mathfrak{g}_{q(1)}$
  makes $w(t)$ an arbitrary element of $\mathfrak{g}_{q(t)}$ at each time.
  Thus $u(t)$ is orthogonal to $\mathfrak{g}_{q(t)}$ for all time.
\end{proof}

At this point, we should return to our example to illustrate this idea.
\begin{Example}
  \label{ex:ex5}
  Again consider the setup of example \ref{ex:two_particles}.
  In this case $q(t) = (q_1(t) , q_2(t) ) \in M \times M - \Delta_{M \times M}$.
  The space $\mathfrak{g}_{z(t)}$ is the space of vector-fields
  which vanish at $q_1(t)$ and $q_2(t)$.
  Therefore,
  $u(t)$ is orthogonal to $q(t)$ if and only if $m = P[u]$ satisfies
  \begin{align*}
    \langle m , v \rangle = p_1 \cdot v(z_1(t)) + p_2 \cdot v(z_2(t))
  \end{align*}
  for some covectors $p_1,p_2$ and
  for any $v \in \mathfrak{X}(M)$.
  In other words
  \begin{align*}
    m = p_1(t) \otimes \delta_{q_1(t)}( \cdot ) + p_2 \otimes \delta_{q_2(t)}(\cdot)
  \end{align*}
  where $\delta_x( \cdot )$ denotes the Dirac delta functional cetnered at $x$.
\end{Example}

This orthogonality constrain allows one to reduce the evolution equation on 
$\mathfrak{X}(M)$ to an evolution equation on $Q$ (which might be finite
dimensional if $G_F$ is large enough).
In particular there is a map $V : TQ \to \mathfrak{X}(M)$
uniquely defined by the conditions
$V(q,\dot{q}) \cdot q = 0$
and $V(q,\dot{q}) \perp \mathfrak{g}_{q}$ with respect to the chosen inner-product
on vector-fields.

\begin{Example}
  Consider the setup of example \ref{ex:two_particles} with $M = \mathbb{R}^n$.
  Then $Q = \mathbb{R}^n \times \mathbb{R}^n - \Delta_{\mathbb{R}^n \times \mathbb{R}^n}$.
  Let $K: \mathbb{R}^n \times \mathbb{R}^n \to \mathbb{R}^{n\times n}$ be the
  matrix-valued reproducing kernel of $P$ (see \cite{MicheliGlaunes2014}).
  Then $V:TQ \to \mathfrak{X}(\mathbb{R}^n)$
  is given by
  \begin{align*}
    V(q,\dot{q}) (x) = K(x - q_1) \cdot p_1 + K(x - q_2) \cdot p_2
  \end{align*}
  where $p_1,p_2 \in \mathbb{R}^n$ are such that $p_1 + K(q_1-q_2) p_2 = \dot{q}_1$
  and $K(q_2-q_1) p_1 + p_2 = \dot{q}_2$.
\end{Example}

One can immediately observe that $V$ is injective and linear in $\dot{q}$.
In other words $V(q,\cdot) : T_qQ \to \mathfrak{X}(M)$ is an injective linear
map for fixed $q \in Q$.
Because the optimal $u(t)$ is orthogonal to $\mathfrak{g}_{q(t)}$
we may invert $V(q(t), \cdot )$ on $u(t)$.
In particular, we may often write the equation of motion on $TQ$ rather than on $\mathfrak{X}(M)$.
This is a massive reduction if $Q$ is finite dimensional.
In particular, the inner-product structure on $\mathfrak{X}(M)$
induces a Riemannian metric on $Q$ given by
\begin{align*}
  g_q(v_1,v_2) = \langle P[ V(q,v_1) ] , V(q,v_2) \rangle.
\end{align*}
The equations of motion in \eqref{eq:extreme2} and \eqref{eq:extreme3}
map to the geodesic equations on $Q$.

\begin{Proposition}
  Let $u$ extremize $E$ or $E_Q$.
  Then there exists a unique trajectory $q(t) \in Q$
  such that $u = V(\frac{dq}{dt})$.
  Moreover, $q(t)$ is a geodesic with respect to the 
  metric $g$.
\end{Proposition}
\begin{proof}
  Let $u$ minimize $E$.  Thus $u$ satisfies \eqref{eq:extreme3}.
  By the previous proposition $u(t)$ is orthogonal to $\mathfrak{g}_{q(t)}$.
  As $V_{q(t)}:T_{q(t)}Q \to \mathfrak{X}(M)$ is injective
  on $\mathfrak{g}_{q(t)}^\perp$, there exists a unique $\dot{q}(t)$
  such that $V(q(t),\dot{q}(t)) = u(t)$.
  Note that $E$ can be written as
  \begin{align*}
    E = \int  \| V(q(t) , \dot{q}(t) ) \| dt + F(q(1))
    = \int g(q,\dot{q} , \dot{q}) ^{1/2} dt + F(q(1)).
  \end{align*}
  Thus, minimizers of $E$ correspond to geodesics in $Q$ with
  respect to the metric $g$.
\end{proof}

If we let $H:T^*Q \to \mathbb{R}$ be the Hamiltonian
induced by the metric on $Q$
we obtain the most data-efficient form or \eqref{eq:extreme2} and \eqref{eq:extreme3}.
Minimizers of $E$ (or $E_Q$) are:
\begin{align}
  \begin{cases}
    (q,p)(t) \in T^*Q \text{ satisfies Hamilton's equations} \\
    p(1) = - DF_Q(q) \\
    q(0) = [ e]_{/G_F}.
  \end{cases}
  \label{eq:extreme4}
\end{align}
We see that this is a boundary value problem posed entirely on $Q$.
If $Q$ is finite dimensional,
this is a massive reduction in terms of data requirements.

\begin{Example}
  Consider the setup of example \ref{ex:two_particles}
  with $M = \mathbb{R}^n$.
  The metric on $Q = M \times M - \Delta_{M^2}$ is
  most easily expressed on the cotangent bundle $T^*Q$.
  If $K$ is the matrix valued kernel of $P$, the metric
  on $T^*Q$ takes the form
  \begin{align*}
    g^*_q( p,p') = \sum_{i,j=1}^2 p_i^T K(q_i - q_j)  p_j'.
  \end{align*}
\end{Example}

\subsection{Descending Group Action}
A related approach to defining distances on a space of objects
to be registered consists of defining an object space $\mathcal{O}$ upon
which $\Diff(M)$ acts transitively\footnote{
This means that for any $o_1,o_2 \in \mathcal{O}$ there exists a $\varphi \in \Diff(M)$
such that $\varphi \cdot o_1 = o_2$
} with distance
\begin{equation*}
  d_{\mathcal{O}}(o_1,o_2)
  =
  \inf_{\varphi\in\Diff(M)}\{d(id,\varphi)\,|\,\varphi\cdot o_1=o_2\}
  \ .
\end{equation*}
Here the distance on $\mathcal{O}$ is defined directly from the distance in
the group that acts on the objects, see for example
\cite{younes_shapes_2010,younes_evolutions_2009}. With this approach, the Riemannian metric
descends from $\Diff(M)$ to a Riemannian metric on $\mathcal{O}$ and geodesics
on $\mathcal{O}$ lift by horizontality to geodesics on $\Diff(M)$.
The quotient spaces $Q$ obtained by reduction by symmetry and their geometric
structure corresponds to the object spaces and geometries defined with this
approach. Intuitively, reduction by symmetry can be considered a removal of
redundant information to obtain compact representations while letting the
metric descend to the object space $\mathcal{O}$ constitutes an approach to defining a
geometric structure on an already known space of objects.
The solutions which result are equivalent to the ones presented in this article
because $\mathcal{O} \cong \Diff(M) / G_o$
where $G_o = \{ \psi \in \Diff(M) \mid \psi(o) = o \}$ for some fixed reference object
$o \in \mathcal{O}$.

\section{Examples}
We here give a number of concrete examples of how symmetry reduce the
infinite dimensional registration problem over $\Diff(M)$ to lower, in some
cases finite, dimensional problems. In all examples, the symmetry of the
dissimilarity measure with respect to a subgroup of $\Diff(M)$ gives a reduced space
by quotienting out the symmetry subgroup.

\subsection{Landmark Matching}
The space $Q$ used in the examples in Section~\ref{sec:red} constitutes a
special case of the landmark matching problem where sets of landmarks
$Q=\{(x_1,\ldots,x_N)|\,x_i\in M,\,x_i\not = x_j\,\forall i\not=j\}$, are placed
into spatial correspondence trough the left action
$\varphi\cdot (x_1,\ldots,x_N)=(\varphi(x_1),\ldots,\varphi(x_N)\}$
of $\Diff(M)$
by minimizing the dissimilarity measure $F(\varphi)=\sum_{i=1}^N\|\varphi(x_i)-x_i\|^2$.
The landmark space $Q$ arises as a quotient of $\Diff(M)$ from the symmetry
group $G_F$ as in in Example~\ref{ex:two_particles}.

Reduction from $\Diff(M)$ to $Q$ in the landmark case has been used in a series
of papers starting with
\cite{joshi_landmark_2000}. Landmark matching is a special case of jet matching
as discussed below. Hamilton's equations \eqref{eq:extreme4} take the form
\begin{align*}
    & \dot{q}_i = \sum_{j=1}^N K(q_i - q_j)  p_j\quad,\quad
     \dot{p}_i = -\sum_{j=1}^N \left( DK(q_i - q_j) p_j \right)^T p_i
\end{align*}
on $T^*Q$ where $DK$ denotes the spatial derivative of the reproducing kernel
$K$. Generalizing the situation in Example~\ref{ex:ex5},
the momentum field is a finite sum of Dirac measures
$\sum_{j=1}^Np_j\otimes\delta_{q_j}$ that through the map $V$ gives an Eulerian velocity field 
as a finite linear combination of the kernel evaluated at $q_i$:
$u(\cdot)=\sum_{j=1}^NK(\cdot-q_j)p_j$.
Registration of landmarks is often in practice done by optimizing over the initial
value of the momentum $p$ in the ODE to minimize $E$, a strategy called shooting
\cite{vaillant_statistics_2004}. Using symmetry, the
optimization problem is thus reduced from an infinite dimensional time-dependent
problem to an $N\dim(M)$ dimensional optimization problem involving integration of a
$2N\dim(M)$ dimensional ODE on $T^*Q$.

\subsection{Curve and Surface Matching}
The space of smooth non-intersecting closed parametrized curves
in $\mathbb{R}^n$ is also known as the space of
embeddings, denoted $\Emb(S^1 , \mathbb{R}^n)$.
The parametrization can be removed by considering the right action of $\Diff(S^1)$ on $\Emb(S^2,\mathbb{R}^n)$
given by
\begin{align*}
  c \in \Emb(S^1, \mathbb{R}^n) \stackrel{\psi \in \Diff(S^1) }{\mapsto} c \circ \psi \in \Emb(S^1, \mathbb{R}^n).
\end{align*}
Then the quotient space $\Gr(S^1, \mathbb{R}^n) := \Emb( S^1 , \mathbb{R}^n ) / \Diff(S^1)$
is the space of \emph{unparametrized curves}.
The space $\Gr(S^1,\mathbb{R}^n)$ is a special case of a nonlinear Grassmannian 
\cite{GayBalmazVizman2012}.  It is not immediately clear if this space is a manifold, although
it is certainly an orbifold.  In fact the same question can be asked of $\Diff(\mathbb{R}^n)$ and $\Emb(S^1,\mathbb{R}^n)$.
A few conditions must be enforced on the space
of embeddings and the space of diffeormophisms
in order to impose a manifold structure on these spaces, and these conditions
along with the metric determine whether or not the quotient $\Gr(S^1,\mathbb{R}^n)$ 
can inherit a manifold structure.
We will not dwell upon these matters here, but instead we refer the reader
to the survey article \cite{BauerBruverisMichor2014}.

When the parametrization is not removed, embedded curves and surfaces can be
matched with the current dissimilarity measure
\cite{vaillant_surface_2005,glaunes_transport_2005}. The objects are considered
elements of the dual of the space $W$ of differential $k$-forms on $M$. In the surface
case, the surface $S$ can be evaluated on a $2$-form $w$ by
\begin{equation}
  S(w)
  =
  \int_S
  w(x)
  (e_x^1,e_x^2)
  d\sigma(x)
  \label{eq:current}
\end{equation}
where $(e_x^1,e_x^2)$ is an orthonormal basis for $T_xS$ and $\sigma$ the
surface element. The dual space $W^*$ is linear an can be equipped with a norm
thereby enabling surfaces to be
compared with the $W^*$ norm. Note that the evaluation \eqref{eq:current}
does not depend on the parametrization of $S$.

The isotropy groups for curves and surfaces generalize the isotropy groups of
landmarks by consisting of warps that keeps the objects fixed, i.e.
\begin{align*}
    G_F \equiv \{ \psi \in \Diff(M) \mid \psi(S) = S \}
    \ .
\end{align*}
The momentum field will be supported on the transported
curves/surfaces $\varphi(t).S$ for optimal paths for $E$ in $\Diff(M)$.
  
%

\subsection{Image Matching}
Images can be registered using either the $L^2$-difference defined in
Example~\ref{ex:imssd} or with other dissimilarity measures such as mutual information 
or correlation ratio \cite{wells_multi-modal_1996,roche_correlation_1998}. The
similarity will be invariant to any infinitesimal deformation orthogonal to the
gradient of dissimilarity measure. In the $L^2$ case, this is equivalent to any infinitesimal deformation
orthogonal to the level lines of the moving image \cite{miller_geodesic_2006}.
The momentum field thus has the form $p(t)=\alpha(t)\nabla\varphi(t).I_0$ for a smooth
function $\alpha(t)$ on $M$ and the registration problem can be reduced to a search over the
scalar field $\alpha(t)$ instead of vector field $p(t)$.

Minimizers for $E$ follow the PDE \cite{younes_evolutions_2009}
\begin{equation}
  \begin{split}
    & \dot{v} = \int_M K(\cdot - y)  \alpha(y)\nabla m(y) dy\quad,\quad
     \dot{m} = -\nabla m^Tv\quad,\quad
     \dot{\alpha} = -\nabla\cdot (\alpha v)
  \end{split}
     \label{eq:impde}
\end{equation}
with $m(t)$ representing the deformed image at time $t$.

\begin{figure}[t]
  \begin{center}
    \includegraphics[width=.30\columnwidth,trim=0 0 0 0,clip]{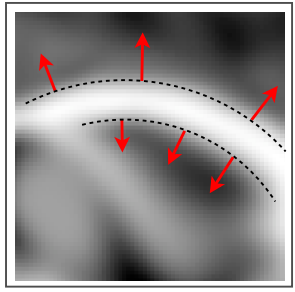}
  \end{center}
  \caption{In image matching, the gradient of the $L^2$-difference will be
    orthogonal to level lines of the image and symmetry implies that the momentum
    field will be orthogonal to the level lines so that
    $p(t)=\alpha(t)\nabla\varphi(t).I_0$ for a time-dependent scalar field $\alpha$.
    }

  \label{fig:imagematch}
\end{figure}

\subsection{Jet Matching}
In \cite{sommer_higher-order_2013,jacobs_symmetries_2013} an extension of the landmark
case has been developed where higher-order information is advected with the
landmarks. These higher-order particles or \emph{jet-particles} have
simultaneously been considered in fluid dynamics \cite{jacobs_coupling_2013},

The spaces of jet particles arise as extensions of the reduced landmark space 
$Q$ by quotienting out smaller isotropy
subgroups. Let $G^{(0)}$ be the isotropy subgroup for a single landmark
\[
	G^{(0)} := \{ \psi \in G \mid \psi(q) = q \}
\]
Let know $k$ be a positive integer. For any $k$-differentiable map $f$ from a neighborhood of $q$, the $k$-jet of
$f$ is denoted $\Jet^{(k)}_{q}( f )$. In coordinates, $\Jet^{(k)}_{q}( f )$
consists of the coefficients of the $k$th order Taylor expansions of $f$ about 
at $x$. The higher-order isotropy subgroups are then given by
\[
	G^{(k)} :=  \{ \psi \in G^{(0)} \mid \Jet_q^{(k)} \psi = \Jet_q^{(k)} \Id \}
    \ .
\]
That is, the elements of $G^{(k)}$ fix the Taylor expansion of the deformation
$\varphi$ up to order $k$. The definition naturally extends to finite number of
landmarks, and the quotients $Q^{(k)}=G/G^{(k)}$ can be identified as the sets
\begin{align*}
      & Q^{(0)} = \{(q_1,\ldots,q_N)\,|\,q_i\in M\} 
    \\
    & Q^{(1)} = \{(
      (q_i^{(0)},q_i^{(1)}),\ldots)\,|\,(q_i^{(0)},q_i^{(1)},\ldots)\in M\times\GL(d)\} 
    \\
    & Q^{(2)} = \{(
      (q_i^{(0)},q_i^{(1)},q_i^{(2)}),\ldots)\,|\,(q_i^{(0)},q_i^{(1)},q_i^{(2)},\ldots)\in M\times\GL(d)\times S^1_2\}
\end{align*}
with $S^1_2$ being the space of rank $(1,2)$ tensors.
Intuitively, the space $Q^{(0)}$ is the regular landmark space with information
about the position of the points; the 1-jet space $Q^{(1)}$ carry for each jet
information about the position and the Jacobian matrix of the warp at the jet
position; and the 2-jet space $Q^{(2)}$ carry in addition the Hessian matrix of 
the warp at the jet position. The momentum for $Q^{(0)}$ in coordinates consists of $N$
vectors representing the local displacement of the points. With the 1-jet space
$Q^{(0)}$, the momentum in addition contains $d\times d$ matrices that can be
interpreted as locally linear deformations at the jet positions
\cite{sommer_higher-order_2013}. In combination with the displacement, the
1-jet momenta can thus be regarded locally affine transformations. The momentum
fields for $Q^{(2)}$ add symmetric tensors encoding local second order
deformation. The local effect effect of the jet particles is sketched in
Figure~\ref{fig:discImage}.

When the dissimilarity measure $F$ is dependent not just on positions but also on
higher-order information around the points, reduction by symmetry implies that 
optimal solutions for $E$ will be parametrized by $k$-jets in the same way as
$Q^{(0)}$ parametrize optimal paths for $E$ in the landmark case.
The higher-order jets can thus be used for landmark matching when the
dissimilarity measure
is dependent on the local geometry around the landmarks. For example, matching of 
first order structure such as image gradients lead to 1-order jets, and matching
of local curvature leads to $2$-order jets. 
%

\subsection{Discrete Image Matching}
The image matching problem can be discretized by evaluating the $L^2$-difference
at a finite number of points. In practice, this alway happens when the integral 
$\int_M|I_0 \circ \varphi^{-1}(x) - I_1(x)|^2dx$ is evaluated at finitely many
pixels of the image. In \cite{sommer_higher-order_2013,jacobs_higher-order_2014}, 
it is shown how this
reduces the image matching PDE \eqref{eq:impde} to a finite dimensional system
on $Q$ when the integral is approximated by pointwise evaluation 
at a grid $\Lambda_h$
\begin{equation}
  F^{(0)} (\varphi) \approx \sum_{x \in \Lambda_h} h^d |I_0(\varphi^{-1}(x) - I_1(x) |^2
  \label{eq:F0}
\end{equation}
where $h>0$ denotes the grid spacing. $F^{(0)}$ approximates $F$ to order
$O(h^d)$, $d=\dim(M)$. The reduced space $Q$ encodes the position of the points
$\varphi^{-1}(x)$, $x\in \Lambda_h$, and the lifted Eulerian momentum field is a
finite sum of point measures
$p=\sum_{x\in\Lambda_h}a_x\otimes\delta_{\varphi^{-1}(x)}$. For each grid
point, the momentum encodes the local displacement of the point, See
Figure~\ref{fig:discImage}.

In \cite{jacobs_higher-order_2014}, a discretization scheme with higher-order 
accuracy is in addition introduced with an $O(h^{d+2})$ approximation $F^{(2)}$
of $F$. The increased accuracy results in the
entire energy $E$ being approximated to order $O(h^{d+2})$. The solution
space in this cases become the jet-space $Q^{(2)}$. For a given order of approximation, a
corresponding reduction in the
number of required discretization points is obtained. The reduction in the
number of discretization points is countered by the increased information
encoded in each 2-jet. The momentum field thus
encodes both local displacement, local linear deformation, and second order
deformation, see Figure~\ref{fig:discImage}.
The discrete solutions will converge to solutions of the non-discretized
problem as $h\rightarrow 0$.
\begin{figure}[t]
  \begin{center}
    \includegraphics[width=.30\columnwidth,trim=0 0 0 0,clip]{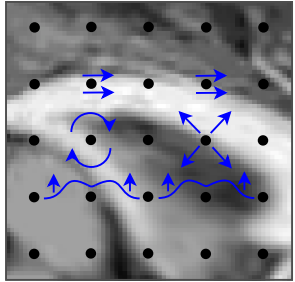}
  \end{center}
  \caption{With discrete image matching, the image is sampled at a regular grid
    $\Lambda_h$, $h>0$ and the image matching PDE \eqref{eq:impde} is reduced to
    an ODE on a
    finite dimensional reduced space $Q$. With the approximation $F^{(0)}$
    \eqref{eq:F0}, the momentum field will encode local displacement
    as indicated by the horizontal arrows (top row). With a first order expansion, the
    solution space will be jet space $Q^{(1)}$ and locally affine motion is
    encoded around each grid point (middle row). The $O(h^{d+2})$ approximation $F^{(2)}$ 
    includes second order information and the system reduces to the jet space
    $Q^{(2)}$ with second order motion encoded at each grid point (lower row).
    }

  \label{fig:discImage}
\end{figure}

\subsection{DWI/DTI Matching}
Image matching is symmetric with respect to variations parallel to the
level lines of the images. With diffusion weighted images (DWI) and the variety
of models for the diffusion information (e.g. diffusion tensor imaging DTI CITE,
Gaussian Mixture Fields CITE), first or 
higher-order information can be reintroduced into the matching problem. In
essence, by letting the dissimilarity measure depend on the diffusion information,
the full $\Diff(M)$ symmetry of the image matching problem is 
reduced to an isotropy subgroup of $\Diff(M)$.

The exact form of the of DWI matching problem depends on the diffusion model and
how $\Diff(M)$ acts on the diffusion image. In \cite{yan_cao_large_2005}, the
diffusion is represented by the principal direction of the diffusion tensor, and
the data objects to be match are thus vector fields. The
action by elements of $\Diff(M)$ is defined by
\begin{equation*}
  \varphi\cdot I(x)
  =
  \frac{\|I\circ\varphi^{-1}\|}{\|D\varphi I\circ\varphi^{-1}\|}D\varphi
  I\circ\varphi^{-1}
  \ .
\end{equation*}
The action rotates the diffusion
vector by the Jacobian of the warp keeping its length fixed. Similar models can
be applied to DTI with the Preservation of Principle Direction scheme (PPD, 
\cite{alexander_strategies_1999,alexander_spatial_2001}) and to GMF based models
\cite{cheng_non-rigid_2009}.
 The dependency on the Jacobian matrix implies that
a reduced model must carry first order information in a similar fashion to the
1-jet space $Q^{(1)}$, however, any irrotational part of the Jacobian can be
removed by symmetry. The full effect of this has yet to be explored.


\subsection{Fluid Mechanics}
Incidentally, the equation of motion
\begin{align*}
  \partial_t m + \pounds_u [m] = 0 \\
  u = K * m
\end{align*}
is an eccentric way of writing Euler's equation
for an invicid incompressible fluid if we assume $u$
is initially in the space of divergence free vector-fields
and $K$ is a Dirac-delta distribution (which impies $m = u$.)
This fact was exploited in \cite{MumfordMichor2013}
to create a sequece of regularized models to Euler's equations
by considering a sequence of Kernels which converge to a Dirac-delta 
distribution.
Moreover, if one replaces $\Diff(M)$ by the subgroup of volume
preserving diffeomorphisms $\Diff_{\rm vol}(M)$,
then (formally) one can produce incompressible particle methods
using the same reduction arguments presented here.
In fact, jet-particles were independently discoverd in this
context as a means of simulating fluids in \cite{jacobs_coupling_2013}.
It is notable that \cite{jacobs_coupling_2013} is a mechanics paper,
and the particle methods which were produced were approached from the perspective
of reduction by symmetry.

In \cite{cotter_jetlet_2014} one of the Kernel parameters in \cite{MumfordMichor2013}
which controls the compressibility of the $u$
was taken to the incompressible limit.
This allowed a realization of the particle methods described in \cite{jacobs_coupling_2013}.
The constructions of \cite{cotter_jetlet_2014} is the same as presented in this survey article,
but with $\Diff(M)$ replaced by the group of 
volume preserving diffeomorphisms of $\mathbb{R}^d$, denoted $\Diff_{\rm vol}( \mathbb{R}^d)$.

\begin{figure}[h]
  \begin{center}
  	\includegraphics[width = 0.2\textwidth]{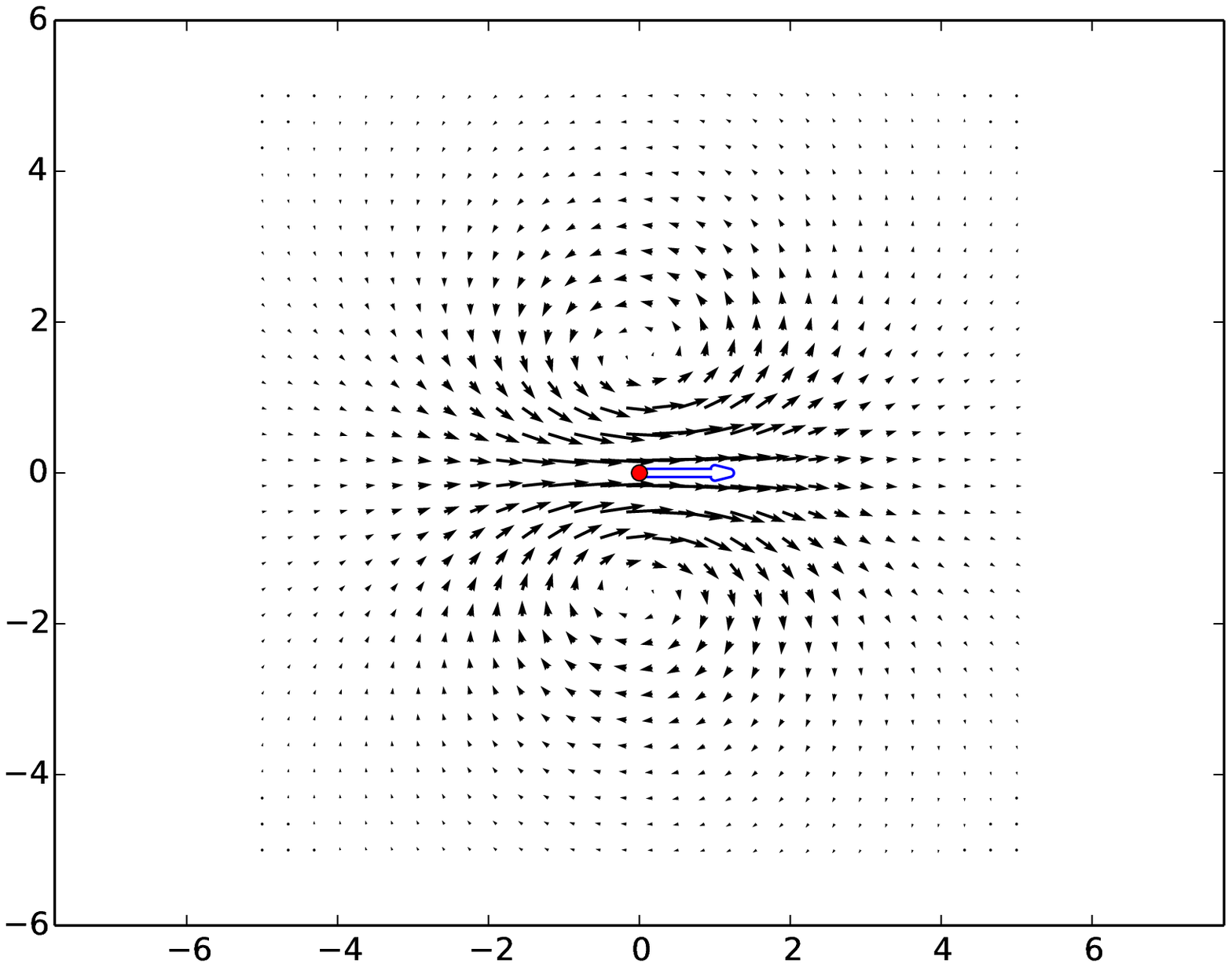}
  	\includegraphics[width = 0.2\textwidth]{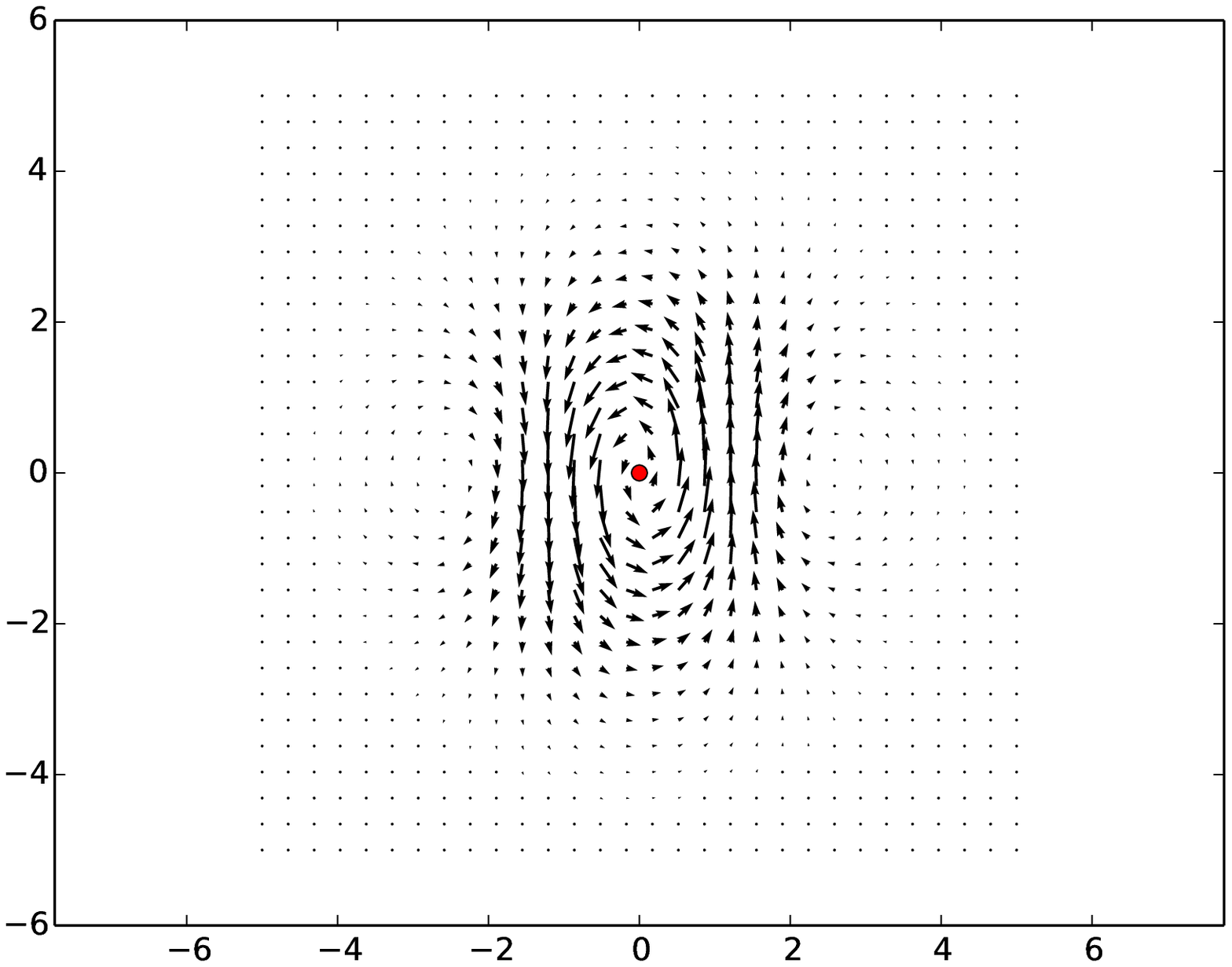}
  	\includegraphics[width = 0.2\textwidth]{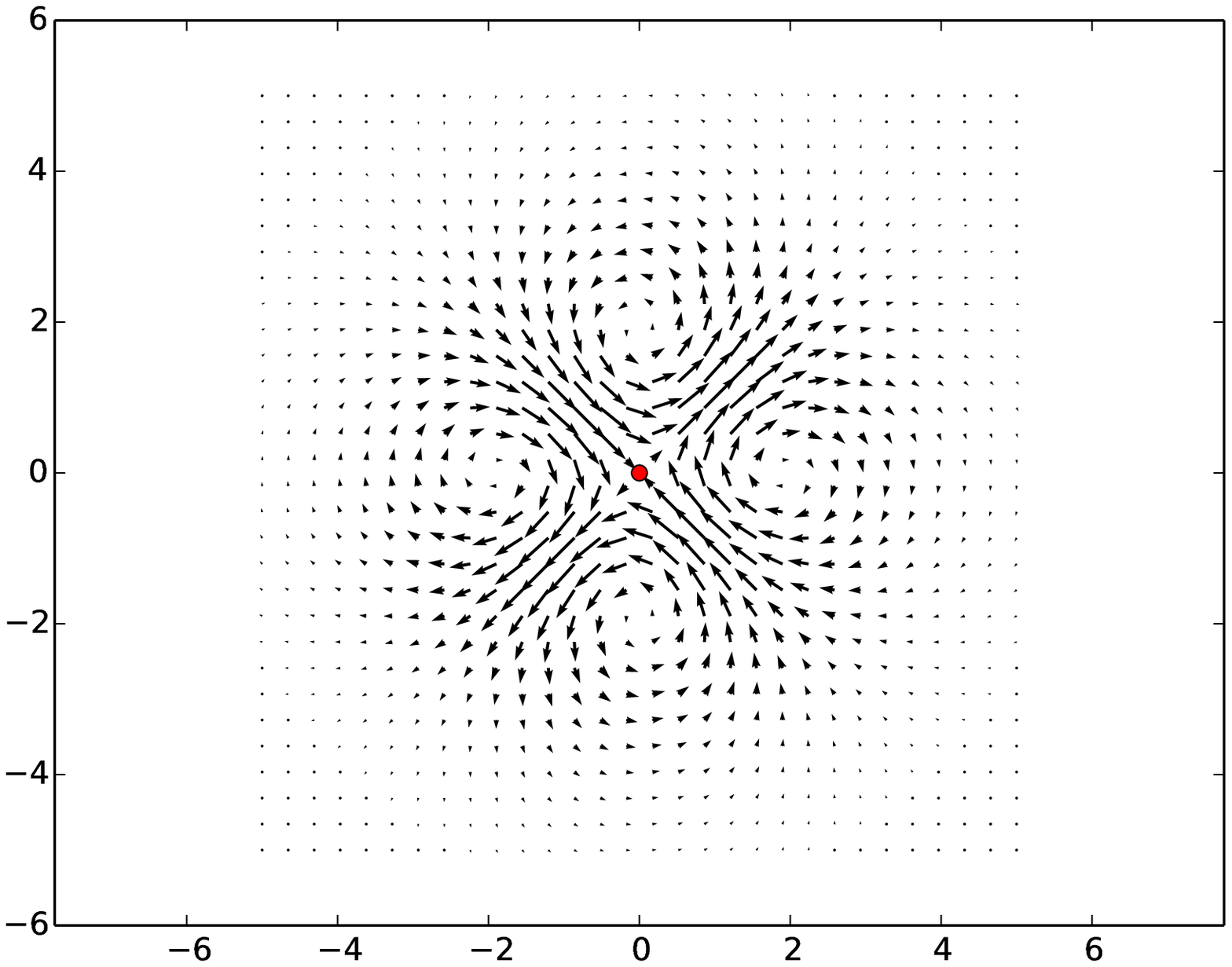}
  	\includegraphics[width = 0.2\textwidth]{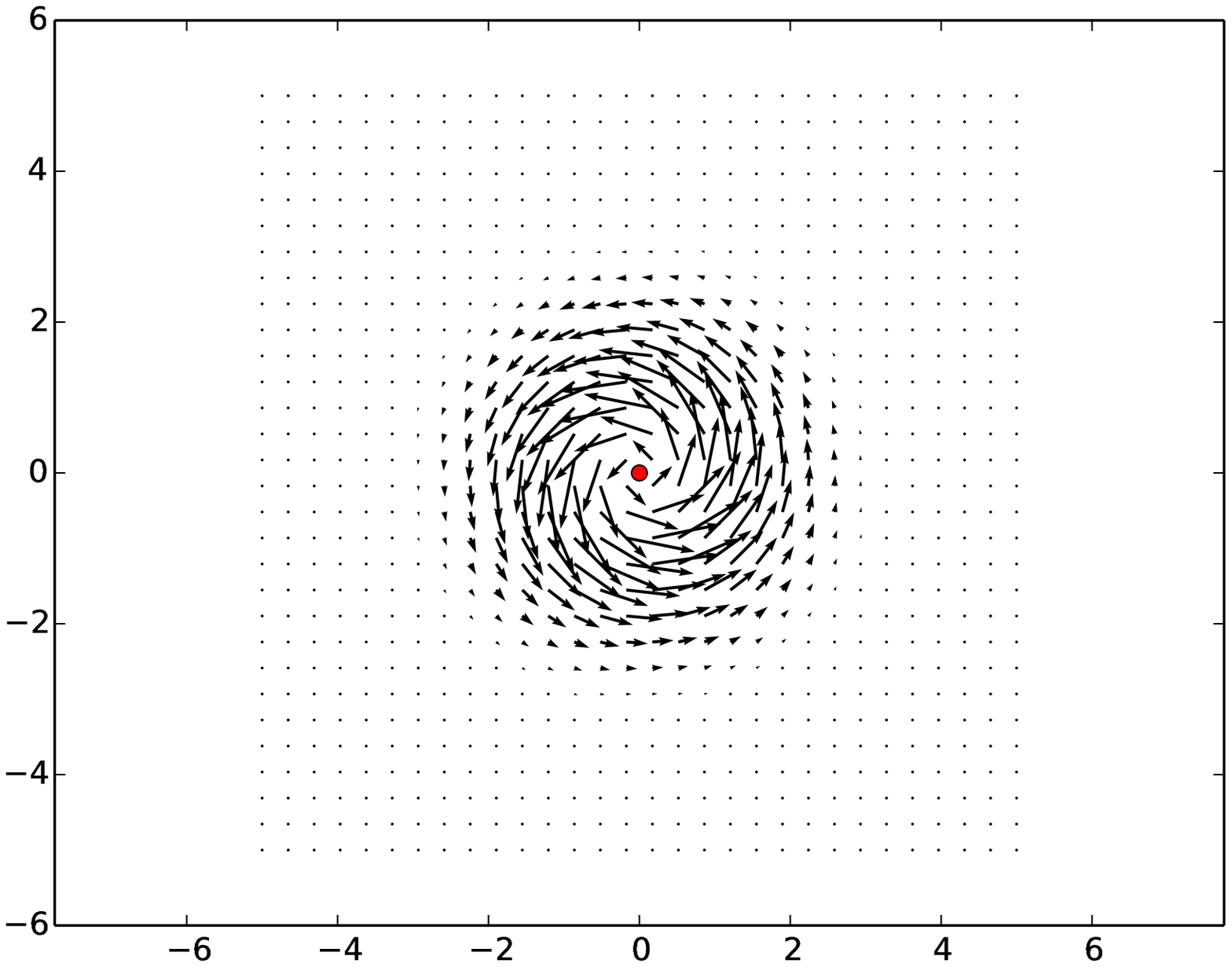}
  \end{center}
  \caption{Velocity fields induced by first order incompressible jet particles
  taken from \cite{cotter_jetlet_2014}
    }
  \label{fig:fluid}
\end{figure}


\section{Discussion and Conclusions}
The information available for solving the registration problem is in practice
not sufficient for uniquely encoding the deformation between the objects to be
registered. Symmetry thus arises in both particle relabeling symmetry
that gives the Eulerian formulation of the equations of motion
and in symmetry groups for specific dissimilarity measures.

For landmark matching, reduction by symmetry reduces the infinite dimensional
registration problem to a finite dimensional problem on the reduced landmark space $Q$. For
matching curves and surfaces, symmetry implies that the momentum stays
concentrated at the curve and surfaces allowing a reduction by the isotropy
groups of warps that leave the objects fixed. In image matching, symmetry allows
reduction by the group of warps that do not change the level sets of the image.
Jet particles have smaller symmetry groups and hence larger reduced
spaces $Q^{(1)}$ and $Q^{(2)}$ that encode locally affine and second order
information.

Reduction by symmetry allow these cases to be handled in one theoretical
framework. We have surveyed the mathematical construction behind the reduction
approach and its relation to the above mentioned examples. As data complexity
rises both in term of resolution an structure, symmetry will continue to be an
important tool for removing redundant information and achieving compact 
data representations.

%
\section{Acknowledgments}
HOJ would like to thank Darryl Holm for providing a bridge from geometric mechanics
into the wonderful world of image registration algorithms.
HOJ is supported by the European Research
Council Advanced Grant 267382 FCCA.
SS is supported by the Danish Council for Independent Research with
the project ``Image Based Quantification of Anatomical Change''.

\bibliographystyle{amsalpha}

\newcommand{\etalchar}[1]{$^{#1}$}
\providecommand{\bysame}{\leavevmode\hbox to3em{\hrulefill}\thinspace}
\providecommand{\MR}{\relax\ifhmode\unskip\space\fi MR }
\providecommand{\MRhref}[2]{%
  \href{http://www.ams.org/mathscinet-getitem?mr=#1}{#2}
}
\providecommand{\href}[2]{#2}

\end{document}